\documentclass[a4paper]{amsart}
\usepackage[margin=1in]{geometry}
\usepackage{graphicx}
\usepackage{amssymb}
\usepackage{amsthm}
\usepackage{mathrsfs}
\usepackage{enumitem}
\usepackage{comment}
\usepackage{color}
\theoremstyle{plain}
\usepackage{appendix}
\usepackage{aligned-overset}

\newtheorem{theorem}{Theorem}
\newtheorem{definition}[theorem]{Definition}
\newtheorem{lemma}[theorem]{Lemma}

\newtheorem{corollary}[theorem]{Corollary}

\newtheorem{remark}[theorem]{Remark}

\newcommand{\Alg}{\mathrm{Alg}} 
\newcommand{\di}{\mathrm{d}} 
\newcommand{\err}{\mathrm{err}} 
\newcommand{\id}{\mathrm{Id}} 
\newcommand{\ie}{\textit{i.e.}} 
\newcommand{\Lip}{\operatorname{Lip}} 
\newcommand{\N}{\mathbb{N}} 
\newcommand{\R}{\mathbb{R}} 
\newcommand{\relu}{\mathrm{ReLU}} 
\newcommand{\supp}{\mathrm{supp}} 
\newcommand{\Z}{\mathbb{Z}} 

\title[]{The sampling complexity of learning invertible residual neural networks}

\author{Yuanyuan Li}
\address{Fudan University, School of Mathematical Sciences, 220 Handan Road, Shanghai, 200433, China }
\email{liyuanyuan20@fudan.edu.cn}

\author{Philipp Grohs}
\address{University of Vienna, Faculty of Mathematics and Research Network Data Science @ Uni Vienna, Kolingasse 14-16,
    1090 Wien}
\email{philipp.grohs@univie.ac.at}

  \author{Philipp Petersen}
\address{University of Vienna, Faculty of Mathematics and Research Network Data Science @ Uni Vienna, Kolingasse 14-16,
    1090 Wien}
\email{philipp.petersen@univie.ac.at}

\date{\today}

\begin{document}

\begin{abstract}
In recent work it has been shown that determining a feedforward ReLU neural network to within high uniform accuracy from point samples suffers from the curse of dimensionality in terms of the number of samples needed. As a consequence, feedforward ReLU neural networks are of limited use for applications where guaranteed high uniform accuracy is required.  

We consider the question of whether the sampling complexity can be improved by restricting the specific neural network architecture. To this end, we investigate invertible residual neural networks which are foundational architectures in deep learning and are widely employed in models that power modern generative methods.
Our main result shows that the residual neural network architecture and invertibility do not help overcome the complexity barriers encountered with simpler feedforward architectures.
Specifically,  we demonstrate that the computational complexity of approximating invertible residual neural networks from point samples in the uniform norm suffers from the curse of dimensionality. Similar results are established for invertible convolutional Residual neural networks.


    
\end{abstract}
\maketitle


\section{Introduction}\label{se:intro}

The central problem of supervised learning is to approximate functions -- which are contained in, or well approximated by a certain model -- based only on information from point samples.
However, for several models currently in use -- for instance feedforward ReLU neural networks \cite{BGV2023} -- the curse of dimensionality arises as the dimension of the input space increases if accuracy in the uniform norm is desired. This means the number of sample points required to accurately approximate the function grows exponentially in the input dimension. 
The curse of dimensionality, therefore, poses significant challenges to achieving good performance in high-dimensional spaces with limited samples.
This manuscript investigates whether imposing specific constraints on the underlying model, such as using invertible residual neural networks and fixing the architecture of neural network models, can reduce the number of sample points needed for accurate approximation. 
We find that, under the uniform norm, this is not the case. 
A lower bound is established on the optimal error for approximating both general and convolutional invertible residual networks, demonstrating the curse of dimensionality on sample point requirements under the uniform norm, even if strong restrictions on the underlying neural network architecture are imposed.

The residual neural network (ResNet) is a crucial architecture used in deep learning, known for addressing the vanishing gradient problem \cite{BSF1994, H1998, BJZP2020} and the degradation problems \cite{HZRS2016} that arise when training very deep neural networks. 
By mitigating these challenges, ResNets make the training more stable and converge better.  
Residual connections are widely adopted in generative artificial networks \cite{GPMXWOCB2014}, image classification \cite{XGDTH2017}, and object detection \cite{RHGS2017}. Additionally, ResNets are fundamental components in popular transformer models \cite{VSPUJGKP2017}, such as GPTs and BERT.  


Invertible neural networks are particularly valuable for tasks requiring reversibility.
Invertible residual neural networks (i-ResNets) \cite{BGCDJ2019}, a variant of traditional residual neural networks, are a type of invertible neural network made possible by specific constraints on the residual blocks. I-ResNets are widely used in various applications, such as generative models \cite{BLLW2022}, 
signal and image processing \cite{XQC2021}, and solving inverse problems \cite{AKWRPKMRK2019}. Convolutional layers \cite{GWKMSSLWWCC2018}, essential components in neural networks, are suited for tasks involving data with locally highly correlated features, such as images and videos \cite{XQC2021, KD2018}. In this manuscript, we will analyze invertible residual neural networks and invertible convolutional neural networks whose residual network blocks are composed of convolutional layers.

An important area of mathematical analysis in deep learning is the study of the expressive power and approximation capabilities of neural networks. 
Classical approximation theory examines how well a target function can be approximated by a simpler function \cite{devore1998nonlinear}. 
The expressivity of shallow neural networks with various activation functions has been extensively studied \cite{C1989, H1991, guhring2022expressivity}. 
Specifically, studies such as \cite{B1993, X2020, EMW2022, SX2022, SX2023} provide convergence rates in terms of the number of neurons in shallow neural networks within various function spaces. 
For deep neural networks, \cite{GKNV2022} explores their expressivity through approximation spaces. Other works, such as \cite{BGKP2019, EPGB2021, yarotsky2017error, petersen2018optimal, lu2021deep}, investigate the connectivity and memory requirements for achieving certain approximation accuracy in deep neural networks. 
In the context of residual neural networks, the universal approximation theory is discussed in \cite{KC2021, LJ2018, TG2023}, while the approximation capabilities of residual networks composed of convolutional blocks are examined in \cite{OS2019}. 
Additionally, approximation results for invertible neural networks are presented in \cite{ZGUA2020}. However, these theories typically do not address the computational complexity of identifying the approximant from limited data. 

From a practical standpoint, \cite{LLMP2023} introduces a regularization method to address derivative approximation with noisy data for shallow neural networks, while the computational complexity of point-sample-based algorithms for approximation of deep neural networks is explored in \cite{GV2023, BGV2023}.

In a similar vein, we seek to determine whether the number of point samples required for approximation can be controlled by imposing additional constraints on the function set, such as fixing the architecture to invertible residual neural networks. This work provides a negative answer. 
Specifically, we establish a lower bound of the optimal error introduced in Section \ref{se:curse} for the approximation of general invertible residual neural networks (Theorem \ref{thm:i-ResNets}) and invertible convolutional residual neural networks (Theorem \ref{thm:i-Conv-ResNets}). 
This lower bound reveals the curse of dimensionality associated with the number of sample points during training, particularly when using the uniform norm. More precisely we show that the minimal number of samples needed to reliably learn a given invertible (convolutional) residual neural network to within a given accuracy in the uniform norm scales exponentially in the input dimension. 
Moreover, this bound applies to all practical training methods, including stochastic gradient descent and its variants. This means that the resulting learning problem is intractable, and additional regularization methods and neural network architectures will have to be developed in order to render accurate (w.r.t. the uniform norm) learning tractable. 

The outline of this manuscript is as follows: Section \ref{se:i-ResNets} gives the basic introduction of i-ResNets and related definitions. Section \ref{se:hat} introduces the hat function \cite{BGV2023} along with variants of this function that we designed to use in invertible (convolutional) residual neural network blocks.  Section \ref{se:curse} uses the properties of the hat function to compute the lower bound on the optimal error for all practical algorithms. Finally, Section \ref{se:summary} summarizes the results and discusses several open problems for future research. 

We conclude this section by introducing the notation that will be used throughout the manuscript. 
\begin{itemize}
    \item For an operator $f$ on a $d$-dimensional domain, we denote by $\Lip(f)$ the Lipschitz constant of $f$ with respect to the $\|\cdot\|_{\infty}$ norm. In our theorems, the $\|\cdot\|_{\infty}$ can be replaced by any other norm with a scaling related to $d$ in the implied constant since we only consider the case $d<\infty$ and the equivalence of norms on finite-dimensional space. 
    \item Let $\sigma:\R\to\R$ be a function and $d\in \N$. For every $z\in\R^{d}$, we define $\sigma(z)$ as $\sigma\left(z\right) = \begin{bmatrix}
        \sigma\left(z_1\right) & \sigma\left(z_2\right) & \cdots & \sigma\left(z_{d}\right)
    \end{bmatrix}$.
    \item All the vectors in this paper are row vectors by default.
    \item $\vec{1} = \begin{bmatrix}
    1 & 1 & \cdots & 1 
\end{bmatrix}_{1\times d}$, $\vec{0} = \begin{bmatrix}
    0 & 0 & \cdots & 0
\end{bmatrix}_{1\times d}$ and $I$ is the identity matrix, $\ie$, $I_{ij}=1$ if and only if $i=j$, otherwise $I_{ij}=0$.  
    \item $\id$ is the identity transform, $\ie$ $\id(x) = x$.
\end{itemize}

\section{Invertible residual neural networks}\label{se:i-ResNets}

Neural networks are inspired by the structure of the human brain, with architectures that can range in complexity. Numerous methods exist for constructing and analyzing these networks, and for a more comprehensive introduction, we refer readers to \cite{devore1998nonlinear, GK2022}. Among the most efficient and commonly employed approaches for building neural networks is the use of layers and blocks. In this work, we introduce two widely used types of layers.
\begin{definition}[Feedforward layers]
    Let $n_1, n_2\in\N$, $W\in \R^{n_2\times n_1}$, $b\in \R^{n_2}$ and a function $\sigma:\R\to \R$. We call $T:\R^{n_1}\to \R^{n_2}$ a feedforward layer with activation function $\sigma$ if for every $x\in\R^{n_1}$,
    \[
    T\left(x\right) = \sigma\left( xW^{T} + b \right).
    \]
    We define the $W$ as the weight of the feedforward layer and the $b$ as the bias of the feedforward layer.  
\end{definition}

\begin{definition}[Convolutional layers]\label{def:convLayer}
    Let $d, n_{1}, n_{2}\in \N$, $b\in \R^{n_{2}}$, a function $\sigma:\R\to \R$, and $\kappa_{i,j}\in \R^d$ for $i=1,\cdots,n_{1}$ and $j=1,\cdots,n_{2}$. 
    We call $T:\left(\R^{d}\right)^{n_1}\to \left(\R^{d}\right)^{n_2}$ a convolutional layer with activation function $\sigma$ if for every $\left\{x_{(i)}\right\}_{i=1}^{n_1}\in \left(\R^{d}\right)^{n_1}$,
    \[
    T\left(\left\{x_{(i)}\right\}_{i=1}^{n_{1}}\right) = \left\{  \sigma\left(\sum_{i=1}^{n_{1}} \kappa_{i,j}*x_{(i)} + b_j\vec{1}\right) \right\}_{j=1}^{n_{2}},
    \]
    where $u*v\in \R^d$ is the convolution of $u\in\R^d$ and $v\in \R^d$, defined as 
    \[
    \left(u*v\right)_i = \sum_{j=1}^{d}u_{i-j}v_{j}, \text{ with }u_{i}=u_{i+kd}, \text{for all }k\in\Z.
    \]
    The data size of the convolutional layer is $d$, the number of input channels of the convolutional layer is $n_{1}$, the number of output channels of the convolutional layer is $n_{2}$, the bias of the convolutional layer is $b$ and the kernels of the convolutional layer are $\left\{\kappa_{i,j}\right\}$ for $i=1,\cdots,n_1$ and $j=1,\cdots, n_2$. 
\end{definition}
Feedforward layers are a fundamental building block of neural networks, while convolutional layers can be seen as a specialized form of feedforward layers with sparse weights and biases. Convolutional layers are particularly effective for image processing tasks due to their beneficial properties such as translation equivariance, which will be discussed in Section \ref{se:hat}. Traditionally, configuring a convolutional layer involves specifying parameters like kernel size, stride, padding type, and pooling \cite{GWKMSSLWWCC2018}. For simplicity, in this work, we use a kernel size equal to the data size, a stride of $1$, periodic padding, and omit pooling. By stacking these layers, we form a multilayer neural network.
\begin{definition}[Multilayer neural networks]
    Let $L\in \N$, $n_{\ell}\in \N$ for $\ell=0,1,\cdots, L+1$ and a function $\sigma:\R\to\R$. Suppose $T_{\ell}:\R^{n_{\ell-1}}\to \R^{n_{\ell}}$ is a feedforward layer with activation function $\sigma$ for $\ell=1,\cdots,L$ and $T_{L+1}:\R^{n_{L}}\to \R^{n_{L+1}}$ is a feedforward layer with activation function $\id$. 
    We call $F:\R^{n_{0}}\to\R^{n_{L+1}}$ a feedforward neural network with activation function $\sigma$ if $F=T_{L+1}\circ T_{L}\circ \cdots \circ T_{1}$. The hidden layers of $F$ are $\left\{T_{\ell}\right\}_{\ell=1}^{L}$, and the output layer of $F$ is $T_{L+1}$. The architecture of $F$ is $\operatorname{Arc_{FNN}}(F)=(n_0,n_1,\cdots,n_{L+1};\sigma)$. 
    
    Correspondingly, let $d, L\in \N$, $n_{\ell}\in \N$ for $\ell=0,1,\cdots, L+1$, and a function $\sigma:\R\to \R$. Suppose $T_{\ell}:\left(\R^{d}\right)^{n_{\ell-1}}\to \left(\R^{d}\right)^{n_{\ell}}$ is a convolutional layer with activation function $\sigma$ for $\ell=1,\cdots,L$ and $T_{L+1}:\left(\R^{d}\right)^{n_{L}}\to \left(\R^{d}\right)^{n_{L+1}}$ is a convolutional layer with activation function $\id$. 
    We call $F:\left(\R^{d}\right)^{n_0}\to \left(\R^{d}\right)^{n_{L+1}}$ a convolutional neural network with activation function $\sigma$ if $F=T_{L+1}\circ T_{L}\circ \cdots \circ T_{1}$. The hidden layers of $F$ are $\left\{T_{\ell}\right\}_{l=\ell}^{L}$, and the output layer of $F$ is $T_{L+1}$. The architecture of $F$ is $\operatorname{Arc_{CNN}}(F)=(d;n_0,n_1,\cdots,n_{L+1};\sigma)$. 
\end{definition}

Training multilayer neural networks becomes increasingly challenging as the number of layers grows \cite{GB2010}, primarily due to the vanishing gradient problem \cite{BSF1994, H1998, BJZP2020}. Residual neural networks, introduced in \cite{HZRS2016}, effectively mitigate this issue by incorporating identity transformations. We define a residual neural network as follows.
\begin{definition}[ResNets]\label{def:ResNet}
    Let $d, L\in\N$. 
    We call $\mathcal{T}:\R^{d}\to \R^{d}$ an $L$-block residual neural network (ResNet) if for every $x\in \R^{d}$,
    \[
    \mathcal{T}(x) = x^{(L)},
    \]
    where $x^{(0)} = x$, 
    \[
    x^{(\ell)} = G_{l}(x^{(\ell-1)}) + x^{(\ell-1)}, \quad \ell=1,\cdots, L,
    \]
    and $G_{\ell}:\R^{d}\to \R^d$ is a map. We call $G_1, \dots, G_L$ the residual functions of $\mathcal{T}$.    
\end{definition}
The choice of $G_{\ell}$ in Definition \ref{def:ResNet} depends on the application. 
For example, in \cite[Figure 2]{HZRS2016}, the map $G_{\ell}$ is a feedforward $\relu$ neural network with a single hidden layer, while in \cite[Figure 3]{HZRS2016}, it is a convolutional neural network with two hidden layers.

Invertible residual neural networks, as introduced in \cite{BGCDJ2019}, can be used for classification, density estimation, and generative modelling within the same model. This means the same network can handle both supervised learning tasks (like image classification) and unsupervised tasks (like generating or estimating data distributions).
We provide a simple definition below.
\begin{definition}[i-ResNets]\label{def:iResNet}
    Let $d,L\in \N$. We call $\mathcal{T}:\R^{d}\to \R^{d}$ an $L$-block invertible residual neural network (i-ResNets) if $\mathcal{T}$ is a ResNet with residual functions $G_1, \dots, G_L$ satisfying that $\Lip\left(G_{\ell}\right)<1$ for $\ell=1,2,\cdots, L$.  
    Specifically, we call the i-ResNet $\mathcal{T}$ an invertible convolutional residual neural network (i-Conv-ResNet) if $G_{\ell}$ is a convolutional neural network with one input channel and one output channel for $\ell=1,\cdots, L$. 
\end{definition}
The condition $\Lip\left(G_{\ell}\right)<1$ in the Definition \ref{def:iResNet} is a sufficient condition for the invertibility of the ResNet block $\left(G_{\ell} + \id\right)$. 

\section{The hat function for the ResNet block}\label{se:hat}
From now on, the function $\sigma$ is fixed to be a ReLU function, $\ie$ $\sigma(x)=\max\left\{0,x\right\}$ for all $x\in \R$. 
In this section, we introduce a hat function and its variants to incorporate into an i-ResNet block and an i-Conv-ResNet block, respectively. 
The hat function, as proposed in \cite[Subsection A.1]{BGV2023}, is considered here with support restricted to a compact domain and with an additional scaling factor. 
\begin{definition}[The hat function]\label{def:hat_func}
    Let $d\in\N$, $c, M>0$ and $t^{*}\in\R$, define
    \begin{equation}\label{eq:Lambda_tMc}
        \Lambda_{t^{*}, M, c}:\R\to \R, \quad t\mapsto \begin{cases}
        0, & t<t^{*}-\frac{1}{M},\\
        c-cM\left|t-t^{*}\right|, & t\ge t^{*}-\frac{1}{M}.
    \end{cases}
    \end{equation}
    Let $z\in\R^d$, define
    \begin{equation}\label{eq:Delta_zMc}
        \Delta_{z, M, c}:\R^{d}\to \R,\quad x\mapsto  \left(\sum_{i=1}^{d} \Lambda_{z_i, M, c}\left(x_{i}\right)\right)-(d-1)c,
    \end{equation}
    and 
    \[
    \vartheta_{z, M, c}:\R^{d}\to \R,\quad x\mapsto \sigma\left( \Delta_{z, M, c}(x) \right).
    \]
    We call $\vartheta_{z, M, c}$ the hat function. 
\end{definition}
The name "hat" vividly describes the shape of the graph of the hat function. It is compactly supported and features a single peak.  
\begin{lemma}\label{prpt:supp_hat}
Let $d\in \N$, $z\in \R^{d}$, and $c,M>0$. The hat function $\vartheta_{z, M,c}:\R^{d}\to \R$ satisfies the following properties,
\begin{enumerate}[label=(\alph*)]
    \item $\supp~\vartheta_{z,M,c} \subset z + \left[-\frac{1}{M},\frac{1}{M}\right]^{d}$ and $0\le \vartheta_{z,M,c}(x)\le c$ for all $x\in \R^{d}$,
    \item $\vartheta_{z,M,c}(x)\ge \frac{c}{2}$ if $x\in z + \left[-\frac{1}{2dM},\frac{1}{2dM}\right]^{d}$.
\end{enumerate}
\end{lemma}
\begin{proof}
    Let $t^{*}\in \R$, 
    define $\Lambda_{t^{*}, M, c}$ as \eqref{eq:Lambda_tMc}. Then $\Lambda_{t^{*},M,c}(t)\le 0$ when $t\notin \left[t^{*}-\frac{1}{M},t^{*}+\frac{1}{M}\right]$ and $\Lambda_{t^{*},M,c}\left(t\right)\le c$ for all $t\in\R$. 
    Thus, if $x\notin z+\left[-\frac{1}{M},\frac{1}{M}\right]^{d}$, there is an index $i$ such that $x_i\notin \left[z_i-\frac{1}{M},z_i+\frac{1}{M}\right]$, $\ie$ $\Lambda_{z_i,M,c}\left(x_i\right)\le 0$. 
    Define $\Delta_{z,M,c}$ as \eqref{eq:Delta_zMc}. 
    Then, 
    \[
    \Delta_{z,M,c}\left(x\right) = 
    \sum_{j\neq i} \Lambda_{z_j,M,c}\left(x_j\right) + \Lambda_{z_i,M,c}\left(x_i\right) - \left(d-1\right)c \le 0.
    \]
    Therefore, $\vartheta_{z,M,c}\left(x\right)=0$. This implies that $\{x \in \R^d \colon \vartheta_{z,M,c}\left(x\right) \neq 0\} \subset z+\left[-\frac{1}{M},\frac{1}{M}\right]^{d}$ and since $z+\left[-\frac{1}{M},\frac{1}{M}\right]^{d}$ is a closed set we have that $\supp~ \vartheta_{z,M,c}\subset z+\left[-\frac{1}{M},\frac{1}{M}\right]^{d}$.  
    At the same time, $\Delta_{z,M,c}\left(x\right)\le dc - \left(d-1\right)c \le c$ if $x\in z+\left[-\frac{1}{M},\frac{1}{M}\right]^{d}$. Thus, $0\le \vartheta_{z,M,c}\left(x\right)\le c$. 

    Additionally, if $x\in z + \left[-\frac{1}{2dM},\frac{1}{2dM}\right]^{d}$, 
    \[
    \Lambda_{z_i,M,c}\left(x_i\right)\ge c-\frac{cM}{2dM} = \frac{\left(2d-1\right)c}{2d}, \quad \text{for all } i=1,2,\cdots,d,
    \] 
    and 
    \[
    \Delta_{z,M,c}\left(x\right)\ge \frac{\left(2d-1\right)c}{2}-\left(d-1\right)c=\frac{c}{2}.
    \]
    Then,     $\vartheta_{z,M,c}\left(x\right)\ge \frac{c}{2}$.
\end{proof}
There is a ReLU feedforward neural network with two hidden layers that represent the hat function. It is established as follows. 
\begin{equation}\label{eq:vartheta_fnn}
    \begin{split}
        \vartheta_{z,M,c}(x) & = \sigma\left(\left(\sigma\left(\left(x-z+M^{-1}\vec{1}\right)cM\right) - \sigma\left(2\left(x-z\right)cM\right)\right)\left(\vec{1}\right)^{T}-(d-1)c\right)\\
        & = \sigma\left( \sigma\left(xW_1^{T}+b_1\right)W_2^{T} + b_2 \right)W_3^{T}+b_3,
    \end{split}
\end{equation}
where 
\begin{equation}\label{eq:params_vartheta}
W_1 = \begin{bmatrix}
    cMI\\
    2cMI
\end{bmatrix}_{2d\times d}, 
W_2 = \begin{bmatrix}
    \vec{1} & -\vec{1}
\end{bmatrix}_{1\times 2d},\\
b_1 = \begin{bmatrix}
    c\vec{1} - cMz & -2cMz
\end{bmatrix}_{1\times 2d},
b_2 = -(d-1)c.
\end{equation}
and 
\[
W_3 = 1,\quad b_3=0.
\]
Moreover, the architecture of $\vartheta_{z,M,c}$ is $\operatorname{Arc_{FNN}}(\vartheta_{z,M,c})=(d,2d,1,1;\sigma)$. 
Next, we will show $\vartheta_{z,M,c}$ is Lipschitz. 
\begin{lemma}\label{prpt:Lip_vartheta} 
    Let $d\in \N$, $z\in \R^{d}$, $c, M>0$, and $\vartheta_{z, M,c}:\R^{d}\to \R$ be the hat function. Then
    $\vartheta_{z,M,c}$ is Lipschitz and $\Lip\left(\vartheta_{z,M,c}\right)\le 3cdM$. 
\end{lemma}
\begin{proof}
We observe that ReLU activation function $\sigma$ is Lipschitz with $\left|\sigma(a)-\sigma(b)\right|\le \left|a-b\right|$ for all $a,b\in \R$.  Let $x$ and $\tilde{x}$ be two points in $\R^d$, then according to the representation \eqref{eq:vartheta_fnn},
\begin{equation*}
    \begin{split}
        \left|\vartheta_{z,M,c}(x)-\vartheta_{z,M,c}(\tilde{x})\right| 
        & \le \left|\left(\sigma\left(xW_1^{T}+b_1\right) - \sigma\left(\tilde{x}W_1^{T}+b_1\right)\right)W_2^{T}\right|\\
        & \le \left\| \sigma\left(xW_1^{T} + b_1\right)-\sigma\left(\tilde{x}W_1^{T} + b_1\right) \right\|_{1}\\
        & \le \left\|\left(x-\tilde{x}\right)W_1^{T}\right\|_{1} = 3cM\|x-\tilde{x}\|_{1}\le 3cdM\|x-\tilde{x}\|_{\infty}.
    \end{split}
\end{equation*}
\end{proof}

The output of the hat function, as defined in Definition \ref{def:hat_func}, is one-dimensional. To ensure compatibility with the i-ResNet architecture, where the dimensionality of the output must match that of the input, we modify the hat function accordingly, preserving the dimensionality throughout the network.
\begin{definition}[The hat function for the i-ResNet block]\label{def:hat_func_i-ResNet}
    Let $d\in \N$, $M>0$, $0<c<\frac{1}{3dM}$, $v=\pm 1$, and $z\in \R^{d}$. 
    We call $\Theta_{z,M,c}^{v}:\R^{d}\to \R^d$ the hat function for the i-ResNet block if for $x\in \R^{d}$,
    \[
    \Theta_{z,M,c}^{v}(x) = \sigma\left(\sigma\left(xW_1^{T} + b_1^{T}\right)W_2^{T} + b_2\right)W_3^{T},
    \]
    where $W_1$, $W_2$, $b_1$ and $b_2$ are the same as in \eqref{eq:params_vartheta}, 
    while $W_{3} = \begin{bmatrix}
        v & 0 & \cdots & 0
    \end{bmatrix}^{T}$.
\end{definition}
Here we represent the hat function for the i-ResNet block by a feedforward neural network with architecture $\operatorname{Arc_{FNN}}(\Theta_{z,M,c}^{v})=(d,2d,1,d;\sigma)$. 
According to the definition, $\Theta_{z,M,c}^{v}(x) = \begin{bmatrix}
    v\vartheta_{z,M,c}(x) & 0 & \cdots & 0
\end{bmatrix}$. Then, we can derive the following properties. 
\begin{lemma}
    Let $d\in \N$, $M>0$, $0<c<\frac{1}{3dM}$, $v=\pm 1$, and $z\in \R^{d}$. Let $\vartheta_{z,M,c}$ be the hat function and $\Theta_{z,M,c}^{v}$ be the hat function for the i-ResNet block. Then,
    \begin{enumerate}[label=(\alph*)]
        \item $\supp~\Theta_{z,M,c}^{v} = \supp~\vartheta_{z,M,c}$ with $0\le \left\|\Theta_{z,M,c}^{v}\left(x\right)\right\|_{\infty}\le c$ for all $x\in \R^{d}$,
        \item $\left\|\Theta_{z,M,c}^{v}\left(x\right)\right\|_{\infty}\ge \frac{c}{2}$ if $x\in z+\left[-\frac{1}{2dM},\frac{1}{2dM}\right]^{d}$,
        \item $\Theta_{z,M,c}$ is Lipschitz and $\Lip\left(\Theta_{z,M,c}^{v}\right) = \Lip\left(\vartheta_{z,M,c}\right) = 3dcM<1$. 
    \end{enumerate}
\end{lemma}

The hat function for the i-ResNet block is implemented as a feedforward neural network with two hidden layers. However, it cannot be represented by a convolutional neural network, as it lacks key properties of convolutional networks, such as translation equivariance, as discussed below.
\begin{definition}[Translation equivariant]\label{def:trans_equi}
    For $d\in\N$, let $R\in \R^{d\times d}$ be defined as 
    \begin{equation}\label{eq:shift}
         R = \begin{bmatrix}
        0 & 0 & \cdots & 0 & 1 \\
        1 & 0 & \cdots & 0 & 0 \\
        0 & 1 & \cdots & 0 & 0 \\
        \vdots & \vdots & \cdots & \vdots & \vdots \\
        0 & 0 & \cdots & 1 & 0
    \end{bmatrix}.
    \end{equation}
    Let the domain $\Omega\subset \R^{d}$ satisfy $\Omega R:=\left\{xR\mid x\in\Omega\right\} \subset \Omega$. The function $F:\Omega \to \R^{d}$ is translation equivariant if $F(x R) = F(x)R$. 
\end{definition}

In Definition \ref{def:trans_equi}, 
the $R$ represents a shift, as for $x\in \R^d$, the $xR = \begin{bmatrix}
    x_2 & x_3 & \cdots & x_d & x_1
\end{bmatrix}$. This shift can be expressed as a convolution, $xR = e*x$, where $e=\begin{bmatrix}
    0  & \cdots & 0 & 1 & 0
\end{bmatrix}$. Convolution is translation equivariant because
\[
\left(u*\left(vR\right) \right)_{i}= \sum_{j=1}^{d}u_{i-j}\left(vR\right)_{j} = \sum_{j=1}^{d}u_{i-j}v_{j+1} = \left(u*v\right)_{i+1} = \left(\left(u*v\right)R\right)_{i}.
\]

Based on the above observation, a convolutional layer with one input channel and one output channel is translation equivariant because of the choice of biases in Definition \ref{def:convLayer}, $\vec{1}R=\vec{1}$, and the fact that the activation function operates coordinate-wise. 
As a result, multilayer convolutional neural networks with one input channel and one output channel also maintain translation equivariance.

Therefore, we need another variant of the hat function for the i-Conv-ResNet block. 
We represent the hat function using three feedforward neural network layers, as shown in \eqref{eq:vartheta_fnn}. The first hidden layer is $T_{1}:\R^{d}\to \R^{2d}$ as
\[
T_{1}\left(x\right) = xW_1^{T} + b_1,\quad x\in \R^{d}
\]
where $W_1$ and $b_1$ is the same as in \eqref{eq:params_vartheta}. 
The second hidden layer is $T_{2}:\R^{2d}\to \R$ as 
\[
T_{2}\left(y\right) = yW_{2}^{T} + b_2,\quad y\in \R^{2d},
\]
where $W_2$ and $b_2$ is the same as in \eqref{eq:params_vartheta}.
The output layer is an identity transform $\id:\R\to \R$.  
We can extend the hat function in Definition \ref{def:hat_func} into a wider feedforward neural network with translation-equivariant properties by putting multiple feedforward networks in parallel. 

Let $d\in \N$, $M>0$, $0<c<\frac{1}{3dM}$, $v=\pm 1$, $z\in \R^{d}$, $R$ as \eqref{eq:shift}, and let 
$\vartheta_{z,M,c}$ be the hat function. Let $W_1$, $W_2$, $b_1$ and $b_2$ be the same as in \eqref{eq:params_vartheta}. 
Define two feedforward layers $\bar{T_1}:\R^{d}\to \R^{2d^2}$ and $\bar{T_2}:\R^{2d^2}\to \R^{d}$ as 
\[
\bar{T_1}\left(x\right) = x\begin{bmatrix}
    W_1^{T} & RW_1^{T} & \cdots & R^{d-1}W_1^{T}
\end{bmatrix}+\begin{bmatrix}
    b_1 & b_1 & \cdots & b_1
\end{bmatrix},\quad x\in \R^{d},
\]
and 
\[
\bar{T_2}\left(y\right) = y\begin{bmatrix}
    W_2^{T} & 0 & \cdots & 0 \\
    0 & W_2^{T} & \cdots & 0 \\
    \vdots & \vdots & \ddots & \vdots \\
    0 & 0 & \cdots  & W_2^{T}
\end{bmatrix} + b_2\cdot \vec{1},\quad y\in \R^{2d^{2}}.
\]
Define a wider feedforward neural network $\Phi_{z,M,c}^{v}:\R^{d}\to \R^{d}$ as 
\begin{equation}\label{eq:Phi}
    \Phi_{z,M,c}^{v}\left(x\right) = \left(v\cdot\id\right)\circ\bar{T_2}\circ\bar{T_1}\left(x\right) = v\begin{bmatrix}
    \vartheta_{z,M,c}\left(x\right) & \vartheta_{z,M,c}\left(xR\right) & \cdots & \vartheta_{z,M,c}\left(xR^{d-1}\right)
\end{bmatrix}, \quad x\in \R^{d}
\end{equation}
According to the definition, the feedforward neural network $\Phi_{z,M,c}^{v}$ is translation equivariant because $\Phi_{z,M,c}^{v}\left(xR\right) = \Phi_{z,M,c}^{v}\left(x\right)R$ for $x\in \R^{d}$ and $R^{d}=I$. 
There is a feedforward neural network representing the first element of $\Phi_{z,M,c}^{v}$, $\ie$ $v\vartheta_{z,M,c}$. Moreover, the architecture satisfies $\operatorname{Arc_{FNN}}(v\vartheta_{z,M,c}) = \operatorname{Arc_{FNN}}(\vartheta_{z,M,c})=(d,2d,1,1;\sigma)$. 
Thus, there is a convolutional neural network representing $\Phi_{z,M,C}^{v}$ according to \cite[Theorem 4.1]{PV2019} with architecture $\operatorname{Arc_{CNN}}(\Phi_{z,M,C}^{v}) =  (d;1,2d,1,1;\sigma)$.

\begin{definition}[The hat function for the i-Conv-ResNet block]\label{def:hat_func_i-Conv-ResNet}
    For $d\in \N$, $M>0$, $0<c<\frac{1}{3dM}$, $v=\pm 1$, $z\in\R^{d}$ and $R$ as \eqref{eq:shift}, the hat function for the i-Conv-ResNet block is $\Phi_{z,M,c}^{v}:\R^{d}\to \R^{d}$ defined as \eqref{eq:Phi} and represented by a convolutional neural network with architecture $\operatorname{Arc_{CNN}}(\Phi_{z,M,c}^{v})=(d;1,2d,1,1;\sigma)$. 
\end{definition}
The support of the hat function for the i-Conv-ResNet block is more complicated than that of the hat function for the i-ResNet block. However, the Lipschitz constants of both functions are the same.
\begin{lemma}\label{prpt:Phi}
    Let $d\in \N$, $M>0$, $0<c<\frac{1}{3dM}$, $v=\pm 1$, $z\in\R^{d}$, $R$ as \eqref{eq:shift}, $\vartheta_{z,M,c}$ the hat function, and $\Phi_{z,M,c}^{v}$ the hat function for the i-Conv-ResNet block. Then,  
    \begin{enumerate}[label=(\alph*)]
        \item\label{em:supp_Phi_up} $\supp~\Phi_{z,M,c}^{v} = \bigcup_{i=1}^{d} \left\{ xR^{i}\mid x\in\supp~\vartheta_{z,M,c}^{v} \right\}\subset \bigcup_{i=1}^{d}zR^{i} + \left[-\frac{1}{M},\frac{1}{M}\right]^{d}$ with $0\le \|\Phi_{z,M,c}^{v}(x)\|_{\infty}\le c$ for all $x\in \R^{d}$,
        \item\label{em:supp_Phi_low} $\|\Phi_{z,M,c}^{v}(x)\|_{\infty}\ge \frac{c}{2}$ if $x\in \bigcup_{i=1}^{d} zR^{i}+\left[-\frac{1}{2dM},\frac{1}{2dM}\right]^{d}$,
        \item\label{em:Lip_Phi} $\Phi_{z,M,c}^{v}$ is Lipschitz and $\Lip\left(\Phi_{z,M,c}^{v}\right)\le 3cdM<1 $. 
    \end{enumerate}
\end{lemma}
\begin{proof}
    According to the definition of $\Phi_{z,M,c}^{v}$, it holds that $\supp~\Phi_{z,M,c}^{v} = \bigcup_{i=1}^{d} \supp~ \vartheta_{z,M,c}^{v}\left(\cdot R^{i}\right)$. 
    For $x\in \R^{d}$, 
    $x\in \supp~ \vartheta_{z,M,c}^{v}\left(\cdot R^{i}\right)$ if and only if $xR^{i}\in \supp~\vartheta_{z,M,c}^{v}$, $\ie$ $x\in \left\{yR^{d-i}\mid y\in \supp~\vartheta_{z,M,c}^{v}\right\}$ with $R^0 =R^{d} = I$. Thus, $\supp~ \vartheta_{z,M,c}^{v}\left(\cdot R^{i}\right) = \left\{yR^{d-i}\mid y\in \supp~\vartheta_{z,M,c}^{v}\right\}$. 
    Together with Lemma \ref{prpt:supp_hat}, \ref{em:supp_Phi_up} and \ref{em:supp_Phi_low} are proved. 

    Choose $x\in \R^{d}$ and $\tilde{x} \in \R^d$. There is an index $i$ such that   
    \begin{equation*}
        \begin{split}
            \left\|\Phi_{z,M,c}^{v}\left(x\right)-\Phi_{z,M,c}^{v}\left(\tilde{x}\right)\right\|_{\infty} 
            & \le |v|\left|\vartheta_{z,M,c}\left(xR^{i}\right)-\vartheta_{z,M,c}\left(\tilde{x}R^{i}\right)\right| \\
            & \le \Lip\left(\vartheta_{z,M,c}\right)\left\|\left(x-\tilde{x}\right)R^{i}\right\|_{\infty}\\
            & = \Lip\left(\vartheta_{z,M,c}\right)\left\|\left(x-\tilde{x}\right)\right\|_{\infty}\\
            &\le 3cdM\left\|x-\tilde{x}\right\|_{\infty},
        \end{split}
    \end{equation*}
    where the last line is due to $\Lip(\vartheta_{z,M,c}) \le 3cdM$ that was proved in Lemma \ref{prpt:Lip_vartheta}.
    Thus, \ref{em:Lip_Phi} is proved.
\end{proof}

\section{Computation complexity with respect to sampling numbers}\label{se:curse}

The learning task involves identifying a suitable function from the hypothesis set that accurately fits the given samples and generalizes well to unseen data \cite{Berner_Grohs_Kutyniok_Petersen_2022}. This task can be approached using various methods, commonly referred to as learning algorithms.
Below, we consider two types of algorithms: deterministic and random.
\begin{definition}[Deterministic algorithms]
    Let $d\in \N$. 
    Given a function set $U\subset \left(C\left([0,1]^{d}\right)\right)^{d} \cap Y$ where $Y$ is a Banach space, and $m\in \N$ point samples $\left\{x_{(i)}, u\left(x_{(i)}\right)\right\}_{i=1}^{m}$ for a function $u\in U$. 
    We call $A:U \to Y$ a deterministic algorithm using $m$ point samples if $A(u)=Q\left(x_{(1)}, \cdots, x_{(m)}, u\left(x_{(1)}\right), \cdots, u\left(x_{(m)}\right)\right)$, where $Q: \left([0,1]^{d}\right)^{m}\times \left(\R^d\right)^{m} \to Y$ is a map from sample points to elements of $Y$. The collection of all deterministic algorithms using $m$ point samples is denoted by $\Alg_m\left(U,Y\right)$.
\end{definition}

In practice, stochastic elements are often incorporated into algorithms to prevent the training process from getting stuck in local minima. For example, stochastic gradient descent (SGD) \cite{RM1951} introduces randomness during neural network training. Then, the overall algorithm can be seen as a collection of deterministic processes governed by specific probability distributions.
\begin{definition}[Random algorithms]\label{def:random_alg}
    Let $d, m\in \N$.
    Given a function set $U\subset \left(C\left([0,1]^{d}\right)\right)^{d} \cap Y$ where $Y$ is a Banach space. 
    Given a probability space $\left(S,\mathcal{F},\mathbb{P}\right)$. We call $\left(\mathbf{A},\mathbf{m}\right)$ a random algorithm using $m$ point samples on average if $\mathbf{A}=\left(A_{s}\right)_{s\in S}$ and $\mathbf{m}:S\to \N$ satisfies the following conditions.
    \begin{enumerate}[label=(\alph*)]
        \item\label{eum:random_alg_m} The map $\mathbf{m}$ is measurable and 
        $\int_{S}\mathbf{m}(s)\di \mathbb{P}(s)\le m$.
        \item For every $s\in S$, the deterministic algorithm $A_{s}\in \Alg_{\mathbf{m}(s)}\left(U,Y\right)$.
    \end{enumerate}
    The collection of all random algorithms using $m$ point samples on average is denoted by $\Alg_{m}^{MC}\left(U,Y\right)$.
\end{definition}
A deterministic algorithm can be viewed as a special random algorithm with $S$ only containing one element. 

We can have the following estimate based on Markov's inequality. 
Suppose $S_0=\left\{s\in S\mid \mathbf{m}(s)\le 2m\right\}$, the condition
$\int_{S}\mathbf{m}(s)\di \mathbb{P}(s)\le m$ gives 
\begin{equation}\label{eq:S_0}
    \mathbb{P}\left(S_0\right)\ge \frac{1}{2},
\end{equation}
since $\int_{S}\mathbf{m}(s)\di \mathbb{P}(s) \ge 2m\mathbb{P}\left(S_0^{C}\right)$. 

The approximation error quantifies how closely the approximation matches the true function. The error can be influenced by factors such as the number of point samples, the complexity of the model, and the dimensionality of the input space. A lower bound on the error indicates the best possible approximation error that can be achieved with a given number of samples. Here, we define an optimal error that quantifies the challenge of reconstructing each function in the function set using a finite number of point samples measured with respect to a given norm. 
\begin{definition}[Optimal error]
    Let $d, m\in \N$. 
    Given a function set $U\subset \left(C\left([0,1]^{d}\right)\right)^{d} \cap Y$ where $Y$ is a Banach space. 
    Given a probability space $\left(S,\mathcal{F},\mathbb{P}\right)$. 
    For the given set $U$, the optimal error $\err_{m}^{MC}\left(U,Y\right)$ is defined as 
    \[
    \err_m^{MC}\left(U,Y\right) = \inf_{\left(\mathbf{A},\mathbf{m}\right)\in \Alg_{m}^{MC}\left(U,Y\right)} \sup_{u\in U} 
    \int_{S}\left\|u-A_s(u)\right\|_{Y}\di \mathbb{P}(s). 
    \]
\end{definition}

As the size of the function set increases, the corresponding optimal error also increases. In the following, we will derive lower bounds for two very restricted function sets: one for i-ResNets and another for i-Conv-ResNets.

To begin with, let us define the norm that will be considered in the following discussion. 
\begin{definition}[$L^{p}\left(\Omega,\R^{d}\right)$-norm]\label{def:Lp_norm}
    Let $d\in \N$, $1\le p\le \infty$, and $\Omega\subset \R^{d}$. 
    For a function $f:\Omega\to \R^{d}$, we define the $L^{p}\left(\Omega,\R^{d}\right)$-norm of $f$ as 
    \[
    \left\|f\right\|_{L^{p}\left(\Omega,\R^{d}\right)} = \left(\int_{\Omega}\left\|f(x)\right\|_{\infty}^{p}\di x \right)^{\frac{1}{p}},
    \]
    when $1\le p<\infty$, and 
    \[
    \left\|f\right\|_{L^{\infty}\left(\Omega,\R^{d}\right)}=\sup_{x\in \Omega} \left\|f(x)\right\|_{\infty},
    \]
    when $p=\infty$. We denote the collection of all functions whose $L^{p}\left(\Omega,\R^{d}\right)$-norm are finite by $L^{p}\left(\Omega,\R^{d}\right)$.  
\end{definition}

The key to establishing the lower bound of the optimal error is that most algorithms cannot distinguish between numerous functions within the function set and a specific target function. First, we will examine the distortion of the domain under an invertible map. For a function $f:\Omega \to \Omega'$ and a set $V\subset \Omega'$, define $f^{-1}\left(V\right) = \left\{x\in \Omega\mid f(x)\in V\right\}$.
\begin{lemma}\label{lem:inv_domain}
    Let $d\in \N$, and $C_1, C_2, D>0$.
    Suppose $f:\R^d \to \R^{d}$ is invertible with $\Lip\left(f\right) \le C_1$ and $\Lip\left(f^{-1}\right)\le C_2$. Let $V = f(x)+[-D,D]^d$. 
    Then
    \[
    x+\left[-C_1^{-1}D,C_1^{-1}D\right]^d \subset f^{-1}\left(V\right)\subset x + \left[-C_2D, C_2D\right]^d.
    \]
\end{lemma}
\begin{proof}
    Suppose $\tilde{x}\in  x+\left[-C_1^{-1}D,C_1^{-1}D\right]^d $, $\ie$ $\|x-\tilde{x}\|_{\infty}\le C_1^{-1}D$. 
    Then, $\|f(x)-f(\tilde{x})\|_{\infty}\le C_1\|x-\tilde{x}\|_{\infty} \le D$, which proves $x+\left[-C_1^{-1}D,C_1^{-1}D\right]^d \subset f^{-1}\left(V\right)$. 
    For every $y\in V$, $f^{-1}(y)$ is well defined since $f$ is invertible. Then $\|x - f^{-1}(y)\|_{\infty} \le \Lip(f^{-1})\|f(x) - y\|_{\infty} \le C_2D$, which proves $f^{-1}\left(V\right)\subset x + \left[-C_2D, C_2D\right]^d$. 
\end{proof}

We can similarly prove and obtain the following corollary.
\begin{corollary}\label{cor:value_domain}
    Let $d\in \N$, $C_1, C_2, D>0$.
    Suppose $f:\R^d \to \R^{d}$ is invertible with $\Lip\left(f\right) \le C_1$ and $\Lip\left(f^{-1}\right)\le C_2$. Let $V=x+[-D,D]^{d}$. Then
    \[
    f(x) + \left[-C_2^{-1}D,C_2^{-1}D\right]^{d}\subset  f(V)\subset f(x) + \left[-C_1D,C_1D\right]^d.
    \]
\end{corollary}

\subsection{The case of i-ResNet}
We use the convention that $\frac{s}{\infty}=0$ for all $s > 0$ in this subsection. 
Let $d\in \N$, $M>0$, $0<c<\frac{1}{3dM}$, $v=\pm 1$ and $z\in \R^{d}$. 
Let us denote the hat function by $\vartheta_{z,M,c}:\R^{d}\to \R$ and the hat function for the i-ReNet block by $\Theta_{z, M,c}^{v}:\R^{d}\to \R^{d}$ as in Definition \ref{def:hat_func_i-ResNet}. We put the hat function $\Theta_{z,M,c}^{v}$ in a residual neural network block, forming a local disturbance on the first variable as 
\begin{equation}\label{eq:hat_func_i-ResNet_block}
    \left(\Theta_{z,M,c}^{v} + \id\right)(x) = \begin{bmatrix}
    x_1 + v\vartheta_{z,M,c}(x) & x_2 & \cdots & x_d
\end{bmatrix},\quad \text{for all }x\in \R^{d}.
\end{equation}

The following lemma estimates the distance between two functions that differ from the target function by the addition or subtraction of the hat function. This result is crucial for the proof of our main theorem.
\begin{lemma}\label{lem:pm_hat}
    Let $d\in \N$, $M>0$, $0<c<\frac{1}{3dM}$, $v=\pm1$, and $C_1,C_2>0$. 
    Suppose $f:\R^{d}\to \R^{d}$ is an invertible function with $\Lip\left(f\right) \le C_1$ and $\Lip\left(f^{-1}\right)\le C_2$.
    Let $z\in f\left(\R^{d}\right) \subset\R^{d}$ and $\Theta_{z,M,c}^{v}$ be the hat function for the i-ResNet block. 
    Let $f_{z,v} = \left(\Theta_{z,M,c}^{v}+\id\right)\left(f\right)$.
    Then for $1\le p\le \infty$,
    \[
    cd^{-\frac{d}{p}}C_1^{-\frac{d}{p}}M^{-\frac{d}{p}} \le \left\|f_{z,+1} -f_{z,-1} \right\|_{L^{p}\left(\R^{d},\R^{d}\right)} \le c 2^{\frac{d}{p}+1}C_2^{\frac{d}{p}}M^{-\frac{d}{p}}. 
    \]
    Moreover, when $p=\infty$,
     \[
    \left\|f_{z,+1}-f_{z,-1}\right\|_{L^{\infty}\left(\R^{d},\R^{d}\right)} = 2c.
    \]
\end{lemma}
\begin{proof}
    Let us recall the hat function $\vartheta_{z,M,c}$. 
    According to the definition of the $L^{p}\left(\R^{d},\R^{d}\right)$-norm, when $1\le p<\infty$, 
    \begin{equation*}
        \begin{split}
            \left\|f_{z,+1} -f_{z,-1} \right\|_{L^{p}\left(\R^{d},\R^{d}\right)} & = \left(\int\left\|f_{z,+1}(x)-f_{z,-1}(x)\right\|_{\infty}^{p}\di x\right)^{\frac{1}{p}} \\
            & = 2\left(\int_{f^{-1}\left(\supp~ \Theta_{z,M,c}^{+1}\right)} \left\| \Theta_{z,M,c}^{+1}\left(f(x)\right) \right\|_{\infty}^{p}\di x \right)^{\frac{1}{p}}\\
            & =2 \left(\int_{f^{-1}\left(\supp~ \vartheta_{z,M,c}\right)} \left|\vartheta_{z,M,c}\left(f\left(x\right)\right)\right|^{p} \di x\right)^{\frac{1}{p}}.
        \end{split}
    \end{equation*}
    One the one hand, let $U_1 = z + \left[-\frac{1}{M},\frac{1}{M}\right]^{d}$ and $V_1 = f^{-1}(z) + \left[-\frac{C_2}{M}, \frac{C_2}{M}\right]^d$.
    Then, 
    \begin{equation*}
        \begin{split}
        \left\|f_{z,+1} -f_{z,-1} \right\|_{L^{p}\left(\R^{d},\R^{d}\right)}  &\le 2 \left(\int_{f^{-1}\left(U_1\right)}\left|\vartheta_{z,M,c}\left(f\left(x\right)\right)\right|^{p} \di x\right)^{\frac{1}{p}}\\
        &\le 2\left( \int_{V_1}\left|\vartheta_{z,M,c}\left(f\left(x\right)\right)\right|^{p}  \di x \right)^{\frac{1}{p}} \\
        &\le c2^{\frac{d}{p}+1}C_2^{\frac{d}{p}}M^{-\frac{d}{p}},
        \end{split}
    \end{equation*}
    where the first line is due to $\supp~\vartheta_{z,M,c}\subset U_1$, the second line is due to $f^{-1}\left(U_1\right)\subset V_1$, and the last line is due to $\left|\vartheta_{z,M,c}(x)\right|\le c$ for every $x\in \R^{d}$. 
    
    On the other hand, let $U_2 = z + \left[-\frac{1}{2dM}, \frac{1}{2dM}\right]^d$ and $V_2=f^{-1}(z) + \left[-\frac{1}{2dC_1M},\frac{1}{2dC_1M}\right]^d$. Then, 
    \begin{equation}\label{eq:low_b_hat_iResNet}
        \begin{split}
        \left\|f_{z,+1} -f_{z,-1} \right\|_{L^{p}\left(\R^{d},\R^{d}\right)}  &\ge 2 \left(\int_{f^{-1}\left(U_2\right)}\left|\vartheta_{z,M,c}\left(f\left(x\right)\right)\right|^{p} \di x\right)^{\frac{1}{p}}\\
        &\ge 2\left( \int_{V_2}\left|\vartheta_{z,M,c}\left(f\left(x\right)\right)\right|^{p}  \di x \right)^{\frac{1}{p}} \\
        &\ge cd^{-\frac{d}{p}}C_1^{-\frac{d}{p}}M^{-\frac{d}{p}},
        \end{split}
    \end{equation}
    where the first line is due to $U_2\subset \supp~\vartheta_{z,M,c}$, the second line is due to $V_2\subset f^{-1}\left(U_2\right)$, and the last line is due to $\vartheta_{z,M,c}\left(x\right)\ge \frac{c}{2}$ when $x\in z+\left[-\frac{1}{2dM},\frac{1}{2dM}\right]^{d}$.

    When $p=\infty$, 
    \[
    \left\|f_{z,+1}-f_{z,-1}\right\|_{L^{\infty}\left(\R^{d},\R^{d}\right)} = \sup_{x\in f^{-1}\left(\supp~ \vartheta_{z,M,c}\right)} 2\left|\vartheta_{z,M,c}\left(f(x)\right)\right|.
    \]
    Then, $\left\|f_{z,+1}-f_{z,-1}\right\|_{L^{\infty}\left(\R^{d},\R^{d}\right)}=2c$ since $\vartheta_{z,M,c}\left(z\right) = c$ and $z\in f\left(\R^{d}\right)$.
\end{proof}

The following theorem illustrates the difficulty in recovering an i-ResNet based on point samples. 


\begin{theorem}\label{thm:i-ResNets}
    Let $d, m\in\N$, $1\le p\le \infty$, and $f:[0,1]^{d}\to f\left([0,1]^{d}\right)\subset \R^{d}$ be a bi-Lipschitz function. 
    Suppose 
    \[
    N = \left\{\mathcal{N}\mid \operatorname{Arc_{FNN}}(\mathcal{\mathcal{N}})=(d,2d,1,d; \sigma), \Lip\left(\mathcal{N}\right)<1, \left\|W\right\|_{\sup}:=\max_{i,j}|W_{ij}|\le 1,  \text{ for all weights $W$ of $\mathcal{N}$}, \right\}
    \]
    and
    \[
    U = \left\{\left(\mathcal{N} + \id\right)\circ f\mid \mathcal{N}\in N \right\}.
    \]
    Then,
    \[
    \err_{m}^{MC}\left(U,L^{p}\left([0,1]^{d},\R^{d}\right)\right)\ge Cm^{-\frac{1}{d}-\frac{1}{p}},
    \]
    where $C>0$ is independent of $m$.
\end{theorem}

\begin{proof}
    Let $C_1, C_2>0$ such that $\Lip\left(f\right) \le C_1$ and $\Lip\left(f^{-1}\right)\le C_2$.  
    Let $x^*=\frac{1}{2}\cdot \vec{1}\in \R^d$. Then $[0,1]^{d} = x^* + \left[-\frac{1}{2},\frac{1}{2}\right]^{d}$, and $f(x^*) + \left[-\frac{1}{2C_2},\frac{1}{2C_2}\right]^d \subset f\left([0,1]^{d}\right) \subset f(x^*) + \left[-\frac{C_1}{2},\frac{C_1}{2}\right]^d$ according to Corollary \ref{cor:value_domain}.

    For $m\in\N$, we can choose the constant 
    \[
    M=2C_2\left\lceil \left(3m\right)^{\frac{1}{d}}\right\rceil,
    \]
    which implies
    \[
    2C_2\cdot \left(3m\right)^{\frac{1}{d}}\le M\le 4C_2\left(3m\right)^{\frac{1}{d}},
    \]
    since $\left(3m\right)^{\frac{1}{d}}\le \left\lceil \left(3m\right)^{\frac{1}{d}}\right\rceil \le 2\left(3m\right)^{\frac{1}{d}}$, 
    and
    \[
    \Gamma = f(x^{*})+\left(\frac{1}{M}-\frac{1}{2C_2}\right)\cdot \vec{1} + \mathcal{G}\subset f(x^{*}) + \left[-\frac{1}{2C_2}, \frac{1}{2C_2}\right]^{d} \subset f\left(\left[0,1\right]^{d}\right),
    \]
    where $\mathcal{G}$ is the set of grid points as 
    \[
    \mathcal{G} = \frac{2}{M}\left\{0,1,\cdots,\left\lceil \left(3m\right)^{\frac{1}{d}}\right\rceil -1\right\}^{d}.
    \]
    Besides, 
    \[
    \Gamma + \left[-\frac{1}{M},\frac{1}{M}\right]^{d}\subset f(x^{*})+\left[-\frac{1}{2C_2},\frac{1}{2C_2}\right]^d\subset f\left([0,1]^{d}\right).
    \]
    The number of elements in $\Gamma$ is 
    \begin{equation}\label{eq:Gamma_num}
        3m \le \#\Gamma = \left(\left\lceil \left(3m\right)^{\frac{1}{d}} \right\rceil\right)^{d}\le 2^{d}\cdot 3m. 
    \end{equation}
     
    For every sampling set $\{\left(x_{(i)}, f(x_{(i)})\right)\}_{i=1}^{n}$ with $x_{(i)}\in [0,1]^{d}$ and $n\le 2m$, let 
    \[
    \tilde{\Gamma} = \left\{ z\in \Gamma\mid z+\left[-\frac{1}{M},\frac{1}{M}\right]^{d} \cap \left\{ f(x_{(i)}\right)\}_{i=1}^{n} = \emptyset \right\}.
    \]
    The number of elements in $\tilde{\Gamma}$ is 
    \begin{equation}\label{eq:tilde_Gamma_num}
        \# \tilde{\Gamma} \ge \# \Gamma - n\ge m.
    \end{equation}

    Let $v=\pm 1$, $c=\frac{1}{6dM}$ and $\Theta_{z,M,c}^{v}$ be the hat function for the i-ResNet block. 
    We consider a set of functions $\left\{ 
    f_{z,v}\mid z\in \Gamma, v=\pm 1 \right\}\subset U$, with $f_{z,v} = \left(\Theta_{z,M,c}^{v}+\id\right)\left(f\right)$. 
    If $z\in \tilde{\Gamma}$, then $f_{z,v}\left(x_{(i)}\right) = f\left(x_{(i)}\right)$ for all $i=1,\cdots,n$ and $v=\pm 1$ because $\supp~\Theta_{z,M,c}^{v}\subset z+\left[-\frac{1}{M},\frac{1}{M}\right]^{d}$, which means $A\left(f_{z,v}\right) = A\left(f\right)$ for every deterministic algorithm $A\in \Alg_{n}\left(U,L^{p}\left(\R^{d},\R^{d}\right)\right)$. 
    
    The mean error of recovering $\left\{f_{z,v}\mid z\in \Gamma, v=\pm 1\right\}$ by $A$ is  
    \begin{equation*}
        \begin{split}
            \frac{1}{2\# \Gamma}\sum_{z\in \Gamma, v\in\pm 1} \left\| f_{z,v} - A\left(f_{z,v}\right) \right\|_{L^{p}\left([0,1]^{d},\R^{d}\right)}
            & = \frac{1}{2\# \Gamma} \sum_{z\in \Gamma} \sum_{v=\pm 1} \|f_{z,v}-A\left(f_{z,v}\right)\|_{L^{p}\left([0,1]^{d},\R^{d}\right)}\\
            & \ge \frac{1}{2\# \Gamma} \sum_{z\in \tilde{\Gamma}} \sum_{v=\pm 1} \left\| f_{z,v}-A\left(f_{z,v}\right) \right\|_{L^{p}\left([0,1]^{d},\R^{d}\right)}\\
            & = \frac{1}{2\# \Gamma}\sum_{z\in \tilde{\Gamma}}  \sum_{v=\pm 1} \left\| f_{z,v}-A\left( f \right) \right\|_{L^{p}\left([0,1]^{d},\R^{d}\right)} \\
            &\ge \frac{1}{2\#\Gamma}\sum_{z\in\tilde{\Gamma}}\left\| f_{z,+1}-f_{z,-1} \right\|_{L^{p}\left([0,1]^{d},\R^{d}\right)}\\
            & \ge \frac{\#\tilde{\Gamma}}{2\#\Gamma} cd^{-\frac{d}{p}}C_1^{-\frac{d}{p}}M^{-\frac{d}{p}} \\
            & \ge \frac{1}{3\cdot2^{d+1}} cd^{-\frac{d}{p}}C_1^{-\frac{d}{p}}M^{-\frac{d}{p}},
        \end{split}
    \end{equation*}
    where the fourth line is due to the triangle inequality, the fifth line is due to Lemma \ref{lem:pm_hat}, and the last line is due to the number of elements in set $\Gamma$ and $\tilde{\Gamma}$, $\ie$ \eqref{eq:Gamma_num} and \eqref{eq:tilde_Gamma_num}. 
    
    Next, we establish the lower bound of the error for a random algorithm $\left(\bf{A}, \bf{m}\right)\in \Alg_{m}^{MC}\left(U, L^{p}\left([0,1]^{d},\R^{d}\right)\right)$. Suppose $\left(\mathbf{A},\mathbf{m}\right)$ is based on a probability space $\left(S,\mathcal{F},\mathbb{P}\right)$. Let $S_0=\left\{s\in S\mid \mathbf{m}\left(s\right)\le 2m\right\}$. Then  
    \begin{equation*}
        \begin{split}
            \sup_{u\in U}
            \int_{S}\left\|u-A_s(u)\right\|_{L^{p}\left([0,1]^{d},\R^{d}\right)} \di \mathbb{P}(s)
            & \ge \frac{1}{2\# \Gamma}\sum_{z\in \Gamma, v\in \pm 1} 
            \int_{S} \left\|f_{z,v}-A_s\left(f_{z,v}\right)\right\|_{L^{p}\left([0,1]^{d},\R^{d}\right)}\di\mathbb{P}(s)\\
            & \ge \frac{1}{2\# \Gamma} \sum_{z\in \Gamma,v\in\pm 1}\int_{S_0} \left\| f_{z,v}-A_s\left(f_{z,v}\right) \right\|_{L^{p}\left([0,1]^{d},\R^{d}\right)}\di \mathbb{P}\left(s\right) \\
            & \ge \frac{1}{3\cdot2^{d+1}} cd^{-\frac{d}{p}}C_1^{-\frac{d}{p}}M^{-\frac{d}{p}}\mathbb{P}\left(S_0\right)\\
            &\ge \frac{1}{3\cdot2^{d+2}} cd^{-\frac{d}{p}}C_1^{-\frac{d}{p}}M^{-\frac{d}{p}}\\
            &= \frac{1}{3^2\cdot2^{d+3}} d^{-\frac{d}{p}-1}C_1^{-\frac{d}{p}}M^{-\frac{d}{p}-1}\\
            &\ge \frac{1}{3^2\cdot2^{d+3}} d^{-\frac{d}{p}-1}C_1^{-\frac{d}{p}}\left(4C_2\right)^{-\frac{d}{p}-1}\left(3m\right)^{-\frac{1}{p}-\frac{1}{d}},
        \end{split}
    \end{equation*}
    where the fourth line is due to \eqref{eq:S_0}, the fifth line is due to the selection of $c=\frac{1}{6dM}$ and the last line is due to the selection of $M\le 4C_2\left(3m\right)^{\frac{1}{d}}$. 
    All random algorithms in $\Alg_{m}^{MC}\left(U, L^{p}\left([0,1]^{d},\R^{d}\right)\right)$ satisfy the above lower bound. Therefore, the optimal error is lower bounded by 
    \[
    \err_{m}^{MC}\left(U,L^{p}\left([0,1]^{d},\R^{d}\right)\right)\ge  \sup_{u\in U}\int_{S}\left\|u-A_s(u)\right\|_{L^{p}\left([0,1]^{d},\R^{d}\right)} \di \mathbb{P}(s)\ge Cm^{-\frac{1}{p}-\frac{1}{d}},
    \]
    where $C>0$ is independent of $m$.
\end{proof}

\subsection{The case of i-Conv-ResNet}
We use the convention that $\frac{s}{\infty}=0$ for all $s>0$ in this subsection. 
We will consider additional constraints where the ResNets are not only invertible but also constructed using convolutional ResNet blocks. 

Let us start with the analysis of the $L^{p}\left(\R^{d},\R^{d}\right)$ distance between two perturbed functions. 
\begin{lemma}\label{lem:pm_Lp_conv}
    Let $d\in \N$, $C_1,C_2, M>0$, $0<c<\frac{1}{3dM}$ and $v=\pm 1$.
    Suppose $g:\R^{d}\to \R^{d}$ is an invertible function with $\Lip\left(g\right) \le C_1$ and $\Lip\left(g^{-1}\right)\le C_2$. 
    Let $z\in g\left(\R^{d}\right)\subset \R^{d}$ and let $\Phi_{z,M,c}^{v}$ be the hat function for the i-Conv-ResNet block. 
    Let $g_{z,v} = \left(\Phi_{z,M,c}^{v}+\id\right)\left(g\right)$.
    Then for $1\le p\le \infty$,
    \[
    cd^{-\frac{d}{p}}C_1^{-\frac{d}{p}}M^{-\frac{d}{p}} \le \left\|g_{z,+1} -g_{z,-1} \right\|_{L^{p}\left(\R^{d},\R^{d}\right)} \le c 2^{\frac{d}{p}+1}dC_2^{\frac{d}{p}}M^{-\frac{d}{p}}. 
    \]
    Moreover, when $p=\infty$,
    \[
    \left\|g_{z,+1}-g_{z,-1}\right\|_{L^{\infty}\left(\R^{d},\R^{d}\right)}=2c. 
    \]
\end{lemma}
\begin{proof}
    We compute
    \begin{equation*}
        \begin{split}
            \left\|g_{z,+1} -g_{z,-1} \right\|_{L^{p}\left(\R^{d},\R^{d}\right)} & = \left(\int\left\|g_{z,+1}(x)-g_{z,-1}(x)\right\|_{\infty}^{p}\di x\right)^{\frac{1}{p}} \\
            & = 2\left(\int_{g^{-1}\left(\supp~ \Phi_{z,M,c}^{+1}\right)} \left\| \Phi_{z,M,c}^{+1}\left(g(x)\right) \right\|_{\infty}^{p}\di x \right)^{\frac{1}{p}}\\
            & \le 2 \left(\int_{g^{-1}\left(\supp~ \Phi_{z,M,c}^{+1}\right)} c^{p}\di x \right)^{\frac{1}{p}},
        \end{split}
    \end{equation*}
    since $\left\|\Phi_{z,M,c}^{+1}\left(z\right)\right\|_{\infty}\le c$ for all $z\in\R^{d}$. 
    If we combine \ref{em:supp_Phi_up} in the Lemma \ref{prpt:Phi}  and Lemma \ref{lem:inv_domain}, 
    \[
    g^{-1}\left(\supp~\Phi_{z,M,c}^{+1}\right) \subset \bigcup_{i=1}^{d} g^{-1}\left(zR^{i}\right) + \left[-\frac{C_2}{M},\frac{C_2}{M}\right]^d.
    \]
    Thus,
    \[
    \|g_{z,+1}-g_{z,-1}\|_{L^{p}\left(\R^{d},\R^{d}\right)}\le c 2^{\frac{d}{p}+1} d C_2^{\frac{d}{p}}M^{-\frac{d}{p}}.
    \] 
    
    Besides, the first element of the hat function for the i-Conv-ResNet block is the same as the hat function, $\ie$ $\left(\Phi_{z,M,c}^{+1}\right)_{1} = \vartheta_{z,M,c}$. 
    Therefore, we use \eqref{eq:low_b_hat_iResNet} in the proof of Lemma \ref{lem:pm_hat} and get 
    \[
    \|g_{z,+1}-g_{z,-1}\|_{L^{p}\left(\R^{d},\R^{d}\right)}\ge cd^{-\frac{d}{p}}C_1^{-\frac{d}{p}}M^{-\frac{d}{p}}.
    \]
\end{proof}

The challenge in recovering i-Conv-ResNets is similar to that of i-ResNets. The proof of the following theorem is nearly identical to the proof of Theorem \ref{thm:i-ResNets}, with the primary differences being the choice of the size of the support $\frac{1}{M}$ and the scaling constant $c$ for the hat function. This is due to the more complex support structure of the hat function for the i-Conv-ResNet block compared to that for the i-ResNet block.
\begin{theorem}\label{thm:i-Conv-ResNets}
    Let $d\in \N$, $1\le p\le \infty$, and $g:[0,1]^{d}\to g\left([0,1]^{d}\right)\subset \R^{d}$ be a bi-Lipschitz function. 
    Suppose 
    \[
    N = \left\{\mathcal{N}\mid \operatorname{Arc_{CNN}}(\mathcal{\mathcal{N}})=(d;1,2d,1,1; \sigma), \Lip\left(\mathcal{N}\right)<1, \left\|\kappa\right\|_{\sup}:=\max_{i=1,2,\cdots,d}|\kappa_{i}|\le 1, \text{ for all kernels $\kappa$ of $\mathcal{N}$}, \right\}
    \]
    and
    \[
    U = \left\{\left(\mathcal{N} + \id\right)\circ g\mid \mathcal{N}\in N\right\}.
    \]
    Then, 
    \[
    \err_{m}^{MC}\left( U,L^{p}\left([0,1]^{d},\R^{d}\right) \right)\ge Cm^{-\frac{1}{d}-\frac{1}{p}}, 
    \]
    where $C>0$ is independent of $m$.
\end{theorem}
\begin{proof}
    Let $C_1, C_2>0$ such that $\Lip\left(g\right) \le C_1$ and $\Lip\left(g^{-1}\right)\le C_2$.  
    Let $x^*=\frac{1}{2}\cdot \vec{1}\in \R^d$. Then $[0,1]^{d} = x^* + \left[-\frac{1}{2},\frac{1}{2}\right]^{d}$, and $g(x^*) + \left[-\frac{1}{2C_2},\frac{1}{2C_2}\right]^d \subset g\left([0,1]^{d}\right) \subset g(x^*) + \left[-\frac{C_1}{2},\frac{C_1}{2}\right]^d$ according to Corollary \ref{cor:value_domain}.

    We can choose constant $M$ as 
    \[
    M=2C_2\left\lceil \left(3dm\right)^{\frac{1}{d}}\right\rceil,
    \]
    that implies
    \[
    2C_2\cdot \left(3dm\right)^{\frac{1}{d}}\le M\le 4C_2\left(3dm\right)^{\frac{1}{d}}.
    \]
    Moreover, 
    \[
    \Gamma = g(x^{*})+\left(\frac{1}{M}-\frac{1}{2C_2}\right)\cdot \vec{1} + \mathcal{G} \subset g\left(x^{*}\right) + \left[-\frac{1}{2C_2},\frac{1}{2C_2}\right]^{d} \subset g\left(\left[0,1\right]^{d}\right),
    \]
    where $\mathcal{G}$ is a set of grid points as 
    \[
    \mathcal{G} = \frac{2}{M}\left\{0,1,\cdots, \left\lceil \left(3dm\right)^{\frac{1}{d}} \right\rceil -1 \right\}^{d}.
    \]
    Besides, 
    $\Gamma+\left[-\frac{1}{M},\frac{1}{M}\right]^{d}\subset g\left(x^{*}\right)+\left[-\frac{1}{2C_2},\frac{1}{2C_2}\right]^{d}\subset g\left([0,1]^{d}\right)$.

    The number of elements in $\Gamma$ satisfies
    $3dm \le \#\Gamma = \left(\left\lceil \left(3dm\right)^{\frac{1}{d}} \right\rceil\right)^{d}\le 2^{d}\cdot3dm$ since $\left(3dm\right)^{\frac{1}{d}}\le \left\lceil \left(3dm\right)^{\frac{1}{d}} \right\rceil \le 2\left(3dm\right)^{\frac{1}{d}}$.
    
    For every sampling set $\{\left(x_{(i)}, g(x_{(i)})\right)\}_{i=1}^{n}$ with $x_{(i)}\in [0,1]^{d}$ and $n\le 2m$, let 
    \begin{equation*}
        \begin{split}
            \tilde{\Gamma} & = \left\{  z\in \Gamma\mid \left(\bigcup_{j=1}^{d} zR^{j}+\left[-\frac{1}{M},\frac{1}{M}\right]^{d} \right)\cap \left\{ g(x_{(i)}\right)\}_{i=1}^{n} = \emptyset \right\} \\
            & = \left\{z\in \Gamma\mid z + \left[-\frac{1}{M},\frac{1}{M}\right]^{d}\cap \left(\bigcup_{j=1}^{d}\left\{g\left(x_{(i)}\right)R^{j}\right\}_{i=1}^{n}\right)=\emptyset \right\}.
        \end{split}
    \end{equation*}
    Then the number of elements in $\tilde{\Gamma}$ satisfies $\# \tilde{\Gamma} \ge \# \Gamma - dn\ge dm$. 
    
    Next, let $c=\frac{1}{6dM}$. We consider a set of functions $\left\{ 
    g_{z,v}\mid z\in \Gamma, v=\pm 1 \right\}\subset U$ with $g_{z,v} = \left(\Phi_{z,M,c}^{v}+\id\right)\left(g\right)$, where $\Phi_{z,M,c}^{v}$ is the hat function for the i-Conv-ResNet block as in the Definition \ref{def:hat_func_i-Conv-ResNet}. 
    Then for every deterministic algorithm $A\in\Alg_{n}\left(U,L^{p}\left(\left[0,1\right]^{d},\R^{d}\right)\right)$, the statement $A(g_{z,v}) = A(g)$ holds when $z\in \tilde{\Gamma}$, since $g_{z,v}\left(x_{(i)}\right) = g\left(x_{(i)}\right)$ for all $i=1,\cdots, n$.  
    The mean error of recovering $\left\{g_{z,v}\mid z\in \Gamma, v=\pm 1\right\}$ by $A$ is  
    \begin{equation*}
        \begin{split}
            \frac{1}{2\# \Gamma}\sum_{z\in \Gamma, v\in\pm 1}\left\|g_{z,v}-A\left(g_{z,v}\right)\right\|_{L^{p}\left(\left[0,1\right]^{d},\R^{d}\right)}
            & \ge \frac{1}{2\#\Gamma} \sum_{z\in\Gamma}\sum_{v\in \pm 1}\left\|g_{z,v}-A\left(g_{z,v}\right)\right\|_{L^{p}\left(\left[0,1\right]^{d}, \R^{d}\right)} \\
            &\ge \frac{1}{2\#\Gamma} \sum_{z\in\tilde{\Gamma}}\sum_{v\in\pm 1}\left\|g_{z,v}-A\left(g_{z,v}\right)\right\|_{L^{p}\left(\left[0,1\right]^{d},\R^{d}\right)}\\
            &\ge \frac{1}{2\#\Gamma}\sum_{z\in\tilde{\Gamma}}\sum_{v=\pm 1}\left\|g_{z,v}-A\left(g\right)\right\|_{L^{p}\left([0,1]^{d},\R^{d}\right)}\\
            &\ge \frac{1}{2\#\Gamma}\sum_{z\in\tilde{\Gamma}}\left\|g_{z,+1}-g_{z,-1}\right\|_{L^{p}\left([0,1]^{d},\R^{d}\right)}\\
            &\ge \frac{\# \tilde{\Gamma}}{2\# \Gamma} cd^{-\frac{d}{p}}C_1^{-\frac{d}{p}}M^{-\frac{d}{p}}\\
            &\ge \frac{1}{3\cdot 2^{d+1}}cd^{-\frac{d}{p}}C_1^{-\frac{d}{p}}M^{-\frac{d}{p}},
        \end{split}
    \end{equation*}
    where the fourth line is due to the triangle inequality, the fifth line is due to Lemma \ref{lem:pm_Lp_conv}, and the last line is due to the number of elements in set $\Gamma$ and $\tilde{\Gamma}$.  
    For a random algorithm $\left(\bf{A}, \bf{m}\right)\in \Alg_{m}^{MC}\left( 
    U, L^{p}\left([0,1]^{d},\R^{d}\right)\right)$, suppose the $\left(\mathbf{A},\mathbf{m}\right)$ is based on a probability space $\left(S,\mathcal{F},\mathbb{P}\right)$. Let $S_0=\left\{s\in S\mid \mathbf{m}\left(s\right)\le 2m\right\}$. Then  
    \begin{equation*}
        \begin{split}
            \sup_{u\in U}
            \int_{S}\left\|u-A_s(u)\right\|_{L^{p}\left([0,1]^{d},\R^{d}\right)}\di \mathbb{P}(s)
            & \ge \frac{1}{2\# \Gamma}
            \sum_{z\in \Gamma, v=\pm 1}
            \int_{S}\left\|g_{z,v}-A_s\left(g_{z,v}\right)\right\|_{L^{p}\left([0,1]^{d},\R^{d}\right)}\di \mathbb{P}(s)\\
            & \ge\frac{1}{2\#\Gamma} \sum_{z\in\Gamma,v\in\pm 1} \int_{S_{0}} \left\|g_{z,v}-A_{s}\left(g_{z,v}\right)\right\|_{L^{p}\left([0,1]^{d},\R^{d}\right)}\di \mathbb{P}\left(s\right)  \\
            & \ge \frac{1}{3\cdot 2^{d+1}}cd^{-\frac{d}{p}}C_1^{-\frac{d}{p}}M^{-\frac{d}{p}}\mathbb{P}\left(S_0\right) \\
            & \ge \frac{1}{3\cdot 2^{d+2}}cd^{-\frac{d}{p}}C_1^{-\frac{d}{p}}M^{-\frac{d}{p}} \\
            &\ge \frac{1}{3^2\cdot2^{d+3}} d^{-\frac{d}{p}-1}C_1^{-\frac{d}{p}}M^{-\frac{d}{p}-1}\\
            &\ge \frac{1}{3^2\cdot2^{d+3}} d^{-\frac{d}{p}-1}C_1^{-\frac{d}{p}}\left(4C_2\right)^{-\frac{d}{p}-1}\left(3dm\right)^{-\frac{1}{p}-\frac{1}{d}},
        \end{split}
    \end{equation*}
    where the fourth line is due to \eqref{eq:S_0}, the fifth line is due to the selection of $c=\frac{1}{6dM}$ and the last line is due to the selection of $M\le 4C_2\left(3dm\right)^{\frac{1}{d}}$. 
    Thus, the optimal error satisfies 
    \[
    \err_{m}^{MC}\left(U,L^{p}\left([0,1]^{d},\R^{d}\right)\right)\ge  \sup_{u\in U}\int_{S}\left\|u-A_s(u)\right\|_{L^{p}\left([0,1]^{d},\R^{d}\right)}\di \mathbb{P}(s) \ge Cm^{-\frac{1}{p}-\frac{1}{d}},
    \]
    where $C>0$ is independent of $m$. 
\end{proof}

\begin{remark}

    The lower bound of the optimal error is scaled by a factor related to the number of point samples $m$ and the dimension domain $d$. This factor is $m^{-\frac{1}{d}}$ when the norm under consideration is the $L^{\infty}\left([0,1]^{d},\R^{d}\right)$-norm. This lower bound, $m^{-\frac{1}{d}}$, highlights the curse of dimensionality, as the number of sample points must increase exponentially to achieve a linear reduction in the optimal error.

    Even under additional constraints, where the ResNets are not only invertible but also constructed using convolutional ResNet blocks, the optimal error remains lower-bounded by $Cm^{-\frac{1}{d}-\frac{1}{p}}$, further indicating the curse of dimensionality in the number of samples, particularly when $p=\infty$.
\end{remark}
\section{Summary and future work}\label{se:summary}

In this manuscript, we discussed the challenges of recovering an i-ResNet or i-Conv-ResNet in relation to the number of sample points. We demonstrated that the optimal error for a very restricted function set of i-ResNets or i-Conv-ResNets is lower-bounded by $Cm^{-\frac{1}{d}-\frac{1}{p}}$, where $m$ represents the average number of point samples used by random algorithms, $d$ denotes the dimension of the input data and $p$ specifies the considered norm, $L^{p}\left([0,1]^{d}, \R^{d}\right)$. 
Since most stochastic training algorithms employed in practice fall within this framework, such errors are generally in practice.

Many interesting topics remain to be explored. We list some of them below.
\begin{enumerate}
    \item Searching for a powerful and practically useful architecture that avoids the $\|\cdot\|_{\infty}$ curse of dimensionality.
    \item Identifying an alternative method to incorporate the hat function into a generative ResNet without altering the number of blocks.
    \item Establishing a specific method to achieve the optimal convergence rate when $p<\infty$.
\end{enumerate}

\section*{Acknowledgement}
Y.~L.\ acknowledges support from the State Scholarship Fund by the China Scholarship Council (CSC), 
P.~P.\ was supported by the Austrian Science Fund (FWF) [P37010]. 

\bibliographystyle{abbrv}
\bibliography{references}

\begin{thebibliography}{10}

\bibitem{AKWRPKMRK2019}
L.~Ardizzone, J.~Kruse, S.~Wirkert, D.~Rahner, E.~Pellegrini, R.~Klessen, L.~Maier-Hein, C.~Rother, and U.~K{\"o}the.
\newblock Analyzing inverse problems with invertible neural networks.
\newblock In {\em International Conference on Learning Representations (ICLR)}, 2019.

\bibitem{B1993}
A.~R. Barron.
\newblock Universal approximation bounds for superpositions of a sigmoidal function.
\newblock {\em IEEE Transactions on Information Theory}, 39(3):930--945, 1993.

\bibitem{BJZP2020}
S.~Basodi, C.~Ji, H.~Zhang, and W.~Pan.
\newblock Gradient amplification: An efficient way to train deep neural networks.
\newblock {\em Big Data Mining and Analytics}, 3(3):196--207, 2020.

\bibitem{BGCDJ2019}
J.~Behrmann, W.~Grathwohl, R.~T.~Q. Chen, D.~Duvenaud, and J.-H. Jacobsen.
\newblock Invertible residual networks.
\newblock In {\em Proceedings of the 36th International Conference on Machine Learning}, volume~97 of {\em Proceedings of Machine Learning Research}, pages 573--582. PMLR, 2019.

\bibitem{BSF1994}
Y.~Bengio, P.~Simard, and P.~Frasconi.
\newblock Learning long-term dependencies with gradient descent is difficult.
\newblock {\em IEEE Transactions on Neural Networks}, 5(2):157--166, 1994.

\bibitem{Berner_Grohs_Kutyniok_Petersen_2022}
J.~Berner, P.~Grohs, G.~Kutyniok, and P.~Petersen.
\newblock {\em The Modern Mathematics of Deep Learning}, page 1–111.
\newblock Cambridge University Press, 2022.

\bibitem{BGV2023}
J.~Berner, P.~Grohs, and F.~Voigtlaender.
\newblock Learning {ReLU} networks to high uniform accuracy is intractable.
\newblock 2023.
\newblock ICLR 2023.

\bibitem{BGKP2019}
H.~B{\"o}lcskei, P.~Grohs, G.~Kutyniok, and P.~Petersen.
\newblock Optimal approximation with sparsely connected deep neural networks.
\newblock {\em SIAM Journal on Mathematics of Data Science}, 1(1):8--45, 2019.

\bibitem{BLLW2022}
S.~Bond-Taylor, A.~Leach, Y.~Long, and C.~Willcocks.
\newblock Deep generative modelling: A comparative review of vaes, gans, normalizing flows, energy-based and autoregressive models.
\newblock {\em IEEE Transactions on Pattern Analysis and Machine Intelligence}, 44(11):7327--7347, 2022.

\bibitem{C1989}
G.~Cybenko.
\newblock Approximation by superpositions of a sigmoidal function.
\newblock {\em Mathematics of Control, Signals and Systems}, 2:303--314, 1989.

\bibitem{devore1998nonlinear}
R.~A. DeVore.
\newblock Nonlinear approximation.
\newblock {\em Acta numerica}, 7:51--150, 1998.

\bibitem{EMW2022}
W.~E, C.~Ma, and L.~Wu.
\newblock The {B}arron space and the flow-induced function spaces for neural network models.
\newblock {\em Constructive Approximation}, 55:369--406, 2022.

\bibitem{EPGB2021}
D.~Elbr{\"a}chter, D.~Perekrestenko, P.~Grohs, and H.~B{\"o}lcskei.
\newblock Deep neural network approximation theory.
\newblock {\em IEEE Transactions on Information Theory}, 67(5):2581--2623, 2021.

\bibitem{GB2010}
X.~Glorot and Y.~Bengio.
\newblock Understanding the difficulty of training deep feedforward neural networks.
\newblock In {\em Proceedings of the Thirteenth International Conference on Artificial Intelligence and Statistics}, volume~9 of {\em Proceedings of Machine Learning Research}, pages 249--256. PMLR, 2010.

\bibitem{GPMXWOCB2014}
I.~Goodfellow, J.~Pouget-Abadie, M.~Mirza, B.~Xu, D.~Warde-Farley, S.~Ozair, A.~Courville, and Y.~Bengio.
\newblock Generative adversarial nets.
\newblock In {\em Advances in Neural Information Processing Systems}, volume~27, pages 2672--2680. Curran Associates, Inc., 2014.

\bibitem{GKNV2022}
R.~Gribonval, G.~Kutyniok, M.~Nielsen, and F.~Voigtlaender.
\newblock Approximation spaces of deep neural networks.
\newblock {\em Constructive Approximation}, 55:259--367, 2022.

\bibitem{GK2022}
P.~Grohs and G.~Kutyniok, editors.
\newblock {\em Mathematical Aspects of Deep Learning}.
\newblock Cambridge University Press, 2022.

\bibitem{GV2023}
P.~Grohs and F.~Voigtlaender.
\newblock Proof of the theory-to-practice gap in deep learning via sampling complexity bounds for neural network approximation spaces.
\newblock {\em Foundations of Computational Mathematics}, 2023.

\bibitem{GWKMSSLWWCC2018}
J.~Gu, Z.~Wang, J.~Kuen, L.~Ma, A.~Shahroudy, B.~Shuai, T.~Liu, X.~Wang, G.~Wang, J.~Cai, and T.~Chen.
\newblock Recent advances in convolutional neural networks.
\newblock {\em Pattern Recognition}, 77:354--377, 2018.

\bibitem{guhring2022expressivity}
I.~G{\"u}hring, M.~Raslan, and G.~Kutyniok.
\newblock Expressivity of deep neural networks.
\newblock {\em Mathematical Aspects of Deep Learning}, page 149, 2022.

\bibitem{HZRS2016}
K.~He, X.~Zhang, S.~Ren, and J.~Sun.
\newblock Deep residual learning for image recognition.
\newblock In {\em 2016 IEEE Conference on Computer Vision and Pattern Recognition (CVPR)}, pages 770--778, Las Vegas, NV, USA, 2016. IEEE.

\bibitem{H1998}
S.~Hochreiter.
\newblock The vanishing gradient problem during learning recurrent neural nets and problem solutions.
\newblock {\em International Journal of Uncertainty, Fuzziness and Knowledge-Based Systems}, 6(2):107--116, 1998.

\bibitem{H1991}
K.~Hornik.
\newblock Approximation capabilities of multilayer feedforward networks.
\newblock {\em Neural Networks}, 4:251--257, 1991.

\bibitem{KD2018}
D.~P. Kingma and P.~Dhariwal.
\newblock Glow: Generative flow with invertible 1x1 convolutions.
\newblock In {\em Advances in Neural Information Processing Systems}, volume~31, pages 10215--10224. Curran Associates, Inc., 2018.

\bibitem{KC2021}
Z.~Kong and K.~Chaudhuri.
\newblock Universal approximation of residual flows in maximum mean discrepancy.
\newblock {\em Preprint}, 2021.

\bibitem{LLMP2023}
Y.~Li, S.~Lu, P.~Math{\'e}, and S.~Pereverzev.
\newblock Two-layer networks with the relu\({}^k\) activation function: Barron spaces and derivative approximation.
\newblock {\em Numerische Mathematik}, 2023.

\bibitem{LJ2018}
H.~W. Lin and S.~Jegelka.
\newblock Resnet with one-neuron hidden layers is a universal approximator.
\newblock In {\em Advances in Neural Information Processing Systems}, volume~31, pages 6169--6178. Curran Associates, Inc., 2018.

\bibitem{lu2021deep}
J.~Lu, Z.~Shen, H.~Yang, and S.~Zhang.
\newblock Deep network approximation for smooth functions.
\newblock {\em SIAM Journal on Mathematical Analysis}, 53(5):5465--5506, 2021.

\bibitem{OS2019}
K.~Oono and T.~Suzuki.
\newblock Approximation and non-parametric estimation of resnet-type convolutional neural networks.
\newblock In {\em Proceedings of the 36th International Conference on Machine Learning}, volume~97 of {\em Proceedings of Machine Learning Research}, pages 4922--4931. PMLR, 2019.

\bibitem{petersen2018optimal}
P.~Petersen and F.~Voigtlaender.
\newblock Optimal approximation of piecewise smooth functions using deep {ReLU} neural networks.
\newblock {\em Neural Networks}, 108:296--330, 2018.

\bibitem{PV2019}
P.~Petersen and F.~Voigtlaender.
\newblock Equivalence of approximation by convolutional neural networks and fully-connected networks.
\newblock {\em Proceedings of the American Mathematical Society}, 148(4):1567--1581, 2020.

\bibitem{RHGS2017}
S.~Ren, K.~He, R.~Girshick, and J.~Sun.
\newblock Faster r-cnn: Towards real-time object detection with region proposal networks.
\newblock {\em IEEE Transactions on Pattern Analysis and Machine Intelligence}, 39(6):1137--1149, 2017.

\bibitem{RM1951}
H.~Robbins and S.~Monro.
\newblock A stochastic approximation method.
\newblock {\em The Annals of Mathematical Statistics}, 22(3):400--407, 1951.

\bibitem{SX2023}
J.~W. Siegel and J.~Xu.
\newblock Characterization of the variation spaces corresponding to shallow neural networks.
\newblock {\em Constructive Approximation}, 57:1109--1132, 2023.

\bibitem{SX2022}
J.~W. Siegel and J.~Xu.
\newblock Sharp bounds on the approximation rates, metric entropy, and n-widths of shallow neural networks.
\newblock {\em Foundations of Computational Mathematics}, 24:481--537, 2024.

\bibitem{TG2023}
P.~Tabuada and B.~Gharesifard.
\newblock Universal approximation power of deep residual neural networks through the lens of control.
\newblock {\em IEEE Transactions on Automatic Control}, 68(5):2715--2728, 2023.

\bibitem{VSPUJGKP2017}
A.~Vaswani, N.~Shazeer, N.~Parmar, J.~Uszkoreit, L.~Jones, A.~N. Gomez, {\L}.~Kaiser, and I.~Polosukhin.
\newblock Attention is all you need.
\newblock In {\em Advances in Neural Information Processing Systems}, volume~30, pages 5998--6008. Curran Associates, Inc., 2017.

\bibitem{XGDTH2017}
S.~Xie, R.~Girshick, P.~Doll{\'a}r, Z.~Tu, and K.~He.
\newblock Aggregated residual transformations for deep neural networks.
\newblock In {\em 2017 IEEE Conference on Computer Vision and Pattern Recognition (CVPR)}, pages 5987--5995, Honolulu, HI, USA, 2017. IEEE.

\bibitem{XQC2021}
Y.~Xing, Z.~Qian, and Q.~Chen.
\newblock Invertible image signal processing.
\newblock In {\em 2021 IEEE/CVF Conference on Computer Vision and Pattern Recognition (CVPR)}, pages 6283--6292, Nashville, TN, USA, 2021. IEEE.

\bibitem{X2020}
J.~Xu.
\newblock Finite neuron method and convergence analysis.
\newblock {\em Communications in Computational Physics}, 28:1707--1745, 2020.

\bibitem{yarotsky2017error}
D.~Yarotsky.
\newblock Error bounds for approximations with deep {ReLU} networks.
\newblock {\em Neural networks}, 94:103--114, 2017.

\bibitem{ZGUA2020}
H.~Zhang, X.~Gao, J.~Unterman, and B.~Poole.
\newblock Approximation capabilities of neural odes and invertible residual networks.
\newblock In {\em Proceedings of the 37th International Conference on Machine Learning}, volume 119 of {\em Proceedings of Machine Learning Research}, pages 11086--11095. PMLR, 2020.

\end{thebibliography}


\begin{thebibliography}{99}
    \bibitem{AG2022} Abdeljawad A, Grohs P (2022) Integral representations of shallow neural network with rectified power unit activation function. Neural Networks 155: 536–550.

    \bibitem{AKWRPKMRK2019} Ardizzone L, Kruse J, Wirkert S, et al. (2019) Analyzing Inverse Problems with Invertible Neural Networks.

    \bibitem{B1993} Barron AR (1993) Universal approximation bounds for superpositions of a sigmoidal function. IEEE Transactions on Information Theory 39: 930–945.

    \bibitem{BJZP2020} Basodi S, Ji C, Zhang H, et al. (2020) Gradient amplification: An efficient way to train deep neural networks. Big Data Mining and Analytics 3: 196–207.

    \bibitem{BGCDJ2019}  Behrmann J, Grathwohl W, Chen RTQ, et al. (2019) Invertible Residual Networks, Proceedings of the 36th International Conference on Machine Learning, PMLR, 573–582.

    \bibitem{BSF1994} Bengio Y, Simard P, Frasconi P (1994) Learning long-term dependencies with gradient descent is difficult. IEEE Transactions on Neural Networks 5: 157–166.

    
    \bibitem{BGV2023}  Berner J, Grohs P, Voigtlaender F (2022) Learning ReLU networks to high uniform accuracy is intractable.

    \bibitem{BLLW2022} Bond-Taylor S, Leach A, Long Y, et al. (2022) Deep Generative Modelling: A Comparative Review of VAEs, GANs, Normalizing Flows, Energy-Based and Autoregressive Models. IEEE Transactions on Pattern Analysis and Machine Intelligence 44: 7327–7347. 

    \bibitem{BGKP2019} Bölcskei H, Grohs P, Kutyniok G, et al. (2019) Optimal Approximation with Sparsely Connected Deep Neural Networks. SIAM Journal on Mathematics of Data Science 1: 8–45.

    \bibitem{C1989} Cybenko G (1989) Approximation by superpositions of a sigmoidal function. Math Control Signal Systems 2: 303–314.

    
    \bibitem{DHP2021} DeVore R, Hanin B, Petrova G (2021) Neural network approximation. Acta Numerica 30: 327–444.

    \bibitem{EMW2022} E W, Ma C, Wu L (2022) The Barron Space and the Flow-Induced Function Spaces for Neural Network Models. Constr Approx 55: 369–406.

    \bibitem{EPGB2021} Elbrächter D, Perekrestenko D, Grohs P, et al. (2021) Deep Neural Network Approximation Theory. IEEE Transactions on Information Theory 67: 2581–2623.

    \bibitem{GB2010} Glorot X, Bengio Y (2010) Understanding the difficulty of training deep feedforward neural networks, Proceedings of the Thirteenth International Conference on Artificial Intelligence and Statistics, JMLR Workshop and Conference Proceedings, 249–256.

    \bibitem{GPMXWOCB2014} Goodfellow I, Pouget-Abadie J, Mirza M, et al. (2014) Generative Adversarial Nets, Advances in Neural Information Processing Systems, Curran Associates, Inc.
    
    \bibitem{GKNV2022} Gribonval R, Kutyniok G, Nielsen M, et al. (2022) Approximation Spaces of Deep Neural Networks. Constr Approx 55: 259–367.

    \bibitem{GK2022}  Grohs P, Kutyniok G (Eds.) (2022) Mathematical Aspects of Deep Learning, Cambridge University Press.

    \bibitem{GV2023} Grohs P, Voigtlaender F (2023) Proof of the Theory-to-Practice Gap in Deep Learning via Sampling Complexity bounds for Neural Network Approximation Spaces. Found Comput Math.

    \bibitem{GWKMSSLWWCC2018} Gu J, Wang Z, Kuen J, et al. (2018) Recent advances in convolutional neural networks. Pattern Recognition 77: 354–377.

    
    \bibitem{HZRS2016}  He K, Zhang X, Ren S, et al. (2016) Deep Residual Learning for Image Recognition, 2016 IEEE Conference on Computer Vision and Pattern Recognition (CVPR), Las Vegas, NV, USA, IEEE, 770–778.

    \bibitem{H1998} Hochreiter S (1998) The Vanishing Gradient Problem During Learning Recurrent Neural Nets and Problem Solutions. Int J Unc Fuzz Knowl Based Syst 06: 107–116.

    \bibitem{H1991} Hornik K (1991) Approximation capabilities of multilayer feedforward networks. Neural Networks 4: 251–257.

    \bibitem{LLMP2023} Li Y, Lu S, Mathé P, et al. (2023) Two-layer networks with the $\text {ReLU}^k$ activation function: Barron spaces and derivative approximation. Numer Math.

    \bibitem{LJ2018} Lin H, Jegelka S (2018) ResNet with one-neuron hidden layers is a Universal Approximator, Advances in Neural Information Processing Systems, Curran Associates, Inc.

    \bibitem{KB2017} Kingma DP, Ba J (2017) Adam: A Method for Stochastic Optimization. arXiv:14126980 [cs]. 

    \bibitem{KD2018} Kingma DP, Dhariwal P (2018) Glow: Generative Flow with Invertible 1x1 Convolutions, Advances in Neural Information Processing Systems, Curran Associates, Inc.

    \bibitem{KC2021} Kong Z, Chaudhuri K (2021) Universal Approximation of Residual Flows in Maximum Mean Discrepancy.
    
    \bibitem{OS2019} Oono K, Suzuki T (2019) Approximation and non-parametric estimation of ResNet-type convolutional neural networks, Proceedings of the 36th International Conference on Machine Learning, PMLR, 4922–4931.

    
    \bibitem{PV2019} Petersen P, Voigtlaender F (2019) Equivalence of approximation by convolutional neural networks and fully-connected networks. Proc Amer Math Soc 148: 1567–1581.

    \bibitem{RM1951} Robbins H, Monro S (1951) A Stochastic Approximation Method. The Annals of Mathematical Statistics 22: 400–407.

    \bibitem{RHGS2017} Ren S, He K, Girshick R, et al. (2017) Faster R-CNN: Towards Real-Time Object Detection with Region Proposal Networks. IEEE Transactions on Pattern Analysis and Machine Intelligence 39: 1137–1149.

    \bibitem{SX2022} Siegel JW, Xu J (2022) Sharp Bounds on the Approximation Rates, Metric Entropy, and n-Widths of Shallow Neural Networks. Found Comput Math.

    \bibitem{SX2023} Siegel JW, Xu J (2023) Characterization of the Variation Spaces Corresponding to Shallow Neural Networks. Constr Approx.
    
    \bibitem{SMDH2013} Sutskever I, Martens J, Dahl G, et al. (2013) On the importance of initialization and momentum in deep learning, Proceedings of the 30th International Conference on Machine Learning, PMLR, 1139–1147.

    \bibitem{TG2023} Tabuada P, Gharesifard B (2023) Universal Approximation Power of Deep Residual Neural Networks Through the Lens of Control. IEEE Trans Automat Contr 68: 2715–2728.

    \bibitem{VSPUJGKP2017} Vaswani A, Shazeer N, Parmar N, et al. (2017) Attention is All you Need, Advances in Neural Information Processing Systems, Curran Associates, Inc.

    \bibitem{XGDTH2017} Xie S, Girshick R, Dollar P, et al. (2017) Aggregated Residual Transformations for Deep Neural Networks, 2017 IEEE Conference on Computer Vision and Pattern Recognition (CVPR), Honolulu, HI, IEEE, 5987–5995. 

    \bibitem{XQC2021} Xing Y, Qian Z, Chen Q (2021) Invertible Image Signal Processing, 2021 IEEE/CVF Conference on Computer Vision and Pattern Recognition (CVPR), Nashville, TN, USA, IEEE, 6283–6292.

    \bibitem{X2020} Xu J (2020) Finite Neuron Method and Convergence Analysis. CiCP 28: 1707–1745.

    \bibitem{ZGUA2020} Zhang H, Gao X, Unterman J, et al. (2020) Approximation Capabilities of Neural ODEs and Invertible Residual Networks, Proceedings of the 37th International Conference on Machine Learning, PMLR, 11086–11095.





\end{thebibliography}

\end{document}